\def\isarxiv{1} 

\ifdefined\isarxiv
\documentclass[11pt]{article}

\usepackage[numbers]{natbib}

\else
\documentclass{article}
\usepackage{neurips_2022}
\fi

\usepackage{amsmath}
\usepackage{amsthm}
\usepackage{amssymb}
\usepackage{algorithm}
\usepackage{subfig}
\usepackage{algpseudocode}
\usepackage{graphicx}
\usepackage{grffile}
\usepackage{wrapfig,epsfig}
\usepackage{url}
\usepackage{xcolor}
\usepackage{epstopdf}

\usepackage{bbm}
\usepackage{dsfont}

\allowdisplaybreaks

\ifdefined\isarxiv

\usepackage{tikz}
\usepackage{hyperref}  
\hypersetup{colorlinks=true,citecolor=blue,linkcolor=blue} 
\usetikzlibrary{arrows}
\usepackage[margin=1in]{geometry}

\else

\usepackage{microtype}
\usepackage{hyperref}
\definecolor{mydarkblue}{rgb}{0,0.08,0.45}
\hypersetup{colorlinks=true, citecolor=mydarkblue,linkcolor=mydarkblue}

\fi
 
\graphicspath{{./figs/}}

\theoremstyle{plain}
\newtheorem{theorem}{Theorem}[section]
\newtheorem{lemma}[theorem]{Lemma}
\newtheorem{definition}[theorem]{Definition}

\newtheorem{proposition}[theorem]{Proposition}

\newtheorem{remark}[theorem]{Remark}

\newcommand{\R}{\mathbb{R}}

\DeclareMathOperator{\poly}{poly}

\DeclareMathOperator{\diag}{diag}

\DeclareMathOperator{\round}{round}

\makeatletter
\newcommand*{\RN}[1]{\expandafter\@slowromancap\romannumeral #1@}
\makeatother

\usepackage{lineno}

\begin{document}

\ifdefined\isarxiv

\date{}

\title{Circuit Complexity Bounds for Visual Autoregressive Model}
\author{
Yekun Ke\thanks{\texttt{ keyekun0628@gmail.com}. Independent Researcher.}
\and
Xiaoyu Li\thanks{\texttt{ 7.xiaoyu.li@gmail.com}. Independent Researcher.}
\and
Yingyu Liang\thanks{\texttt{
yingyul@hku.hk}. The University of Hong Kong. \texttt{
yliang@cs.wisc.edu}. University of Wisconsin-Madison.} 
\and
Zhenmei Shi\thanks{\texttt{
zhmeishi@cs.wisc.edu}. University of Wisconsin-Madison.}
\and
Zhao Song\thanks{\texttt{ magic.linuxkde@gmail.com}. The Simons Institute for the Theory of Computing at UC Berkeley.}
}

\else

\title{Intern Project} 
\maketitle 
\fi

\ifdefined\isarxiv
\begin{titlepage}
  \maketitle
  \begin{abstract}
Understanding the expressive ability of a specific model is essential for grasping its capacity limitations. Recently, several studies have established circuit complexity bounds for Transformer architecture. Besides, the Visual AutoRegressive (VAR) model has risen to be a prominent method in the field of image generation, outperforming previous techniques, such as Diffusion Transformers, in generating high-quality images. We investigate the circuit complexity of the VAR model and establish a bound in this study. Our primary result demonstrates that the VAR model is equivalent to a simulation by a uniform $\mathsf{TC}^0$ threshold circuit with hidden dimension $d \leq O(n)$ and $\mathrm{poly}(n)$ precision. This is the first study to rigorously highlight the limitations in the expressive power of VAR models despite their impressive performance. We believe our findings will offer valuable insights into the inherent constraints of these models and guide the development of more efficient and expressive architectures in the future.

  \end{abstract}
  \thispagestyle{empty}
\end{titlepage}

{\hypersetup{linkcolor=black}
\tableofcontents
}
\newpage

\else

\begin{abstract}
Understanding the expressive ability of a specific model is essential for grasping its capacity limitations. Recently, several studies have established circuit complexity bounds for Transformer architecture. Besides, the Visual AutoRegressive (VAR) model has risen to be a prominent method in the field of image generation, outperforming previous techniques, such as Diffusion Transformers, in generating high-quality images. We investigate the circuit complexity of the VAR model and establish a bound in this study. Our primary result demonstrates that the VAR model is equivalent to a simulation by a uniform $\mathsf{TC}^0$ threshold circuit with hidden dimension $d \leq O(n)$ and $\mathrm{poly}(n)$ precision. This is the first study to rigorously highlight the limitations in the expressive power of VAR models despite their impressive performance. We believe our findings will offer valuable insights into the inherent constraints of these models and guide the development of more efficient and expressive architectures in the future.
\end{abstract}

\fi

\section{Introduction}

Visual generation has seen widespread applications across various domains, including image restoration  \cite{lhc+25,gld+25}, augmented reality \cite{awt+24}, medical imaging \cite{akh+24,mhl+24,lll+24}, and creative industries such as game development \cite{rhr+20, cgx+25}. By generating realistic and diverse images from textual descriptions or other forms of input, visual generation models are transforming how machines perceive and produce visual content. Among the most popular models for visual generation are Variational AutoEncoders (VAE) \cite{doe16}, Generative Adversarial Networks (GAN) \cite{gpm+20},  Diffusion models \cite{swmg15,hja20}, and Flow-based models \cite{kd18}. These models have made notable progress in producing high-quality, high-resolution, and diverse images, expanding the potential of visual generation through improvements in realism, diversity, and fidelity.

However, the introduction of the Visual AutoRegressive model (VAR) \cite{tjy+24} represents a significant shift in the paradigm in this field. Instead of the traditional ``next-token prediction'', the VAR model adopts a coarse-to-fine ``next-scale prediction'' approach. Through this innovative approach, the VAR model is able to capture visual distributions more effectively, exceeding the performance of diffusion transformers in image generation tasks.  Additionally, VAR’s zero-shot generalization capability spans multiple tasks, including image inpainting and manipulation. These results suggest that VAR  offers a promising direction for autoregressive models in visual generation.

As the VAR model demonstrates its impressive performance, it is crucial to explore the limitations of the expressiveness of the VAR model. Up to now, the expressiveness from a circuit complexity perspective of the VAR model remains underexplored. This gap raises an important question:
\begin{center}
   {\it What are the limitations of the expressive power of the VAR model in terms of circuit complexity?} 
\end{center}

To explore this issue, we apply circuit complexity theory, which offers valuable tools for analyzing the computational resources needed for specific tasks. By representing the VAR model as complexity circuits, we can systematically evaluate their capabilities and determine the lower bounds of the problems they can address.

In this work, we present a comprehensive theoretical investigation into the circuit complexity bounds of the VAR models. Our approach involves analyzing and formulating the architecture of the VAR model and analyzing the computational complexity of its components, such as up-interpolation layers, convolution layers, transformer blocks, etc. Finally, we show that uniform $\mathsf{TC}^0$ circuits can efficiently simulate these models.

The primary contributions of our work are summarized below:

\begin{itemize}
    \item As far as we know, this is the first paper to present a mathematical formulation of the Visual AutoRegressive model (Section~\ref{sec:model_formulation}).
    \item We prove that $\mathsf{DLOGTIME}$-uniform $\mathsf{TC}^0$ circuit family can simulate any Visual AutoRegressive model with $O(1)$ depth, $\poly(n)$ size, and $\poly(n)$ precision (Theorem~\ref{thm:main_theorem}).
\end{itemize}

{\bf Roadmap.} Section~\ref{sec:related_work} offers a summary of the related works. Section~\ref{sec:pre} introduces the necessary notations and definitions for the subsequent analysis. In Section~\ref{sec:model_formulation}, we present the mathematical formulation of the VAR model. Section~\ref{sec:complexity_result} details the circuit complexity results for the VAR model. Section~\ref{sec:conclusion} presents the conclusions of our work.


\section{Related Work}\label{sec:related_work}
\subsection{Circuit Complexity and Neural Network} \label{sec:rw_complexity_and_nn}
In computational theory, circuit complexity \cite{ab09} refers to the classification and analysis of computational problems based on the size and depth of Boolean circuits required to solve them, aiming to understand the inherent difficulty of problems in terms of circuit resources. Crucial to the study of computational complexity is the class $\mathsf{AC}^0$, which consists of decision problems solvable by constant-depth Boolean circuits with unbounded fan-in and the logic gates $\mathsf{AND}$, $\mathsf{OR}$, and $\mathsf{NOT}$; $\mathsf{TC}^0$, which extends $\mathsf{AC}^0$ by incorporating $\mathsf{MAJORITY}$ gates; $\mathsf{NC}^1$ represents the class of problems solvable in parallel by circuits with a depth of $O(\log(n))$, where the gate arity is bounded. These three complexity classes form a hierarchical structure: $\mathsf{AC}^0 \subset \mathsf{TC}^0 \subseteq \mathsf{NC}^{1}$ \cite{mss22}. However, the question of whether $\mathsf{TC}^0 = \mathsf{NC}^0$ remains unresolved.

Circuit Complexity has important applications in understanding the capabilities of deep learning models \cite{pmb19,hah20,lag+22,haf22,mss22,ms23,fzg24,cll+24,lll+24_hop,hlsl24,cll+24_rope,lls+24_tensor}. Specifically, \cite{hah20} investigates the computational boundaries of self-attention, demonstrating that, despite its effectiveness in NLP tasks, it has difficulty modeling periodic finite-state languages and hierarchical structures without scaling up the number of layers or attention heads.
\cite{fzg24} delves into the theoretical underpinnings of Chain-of-Thought (CoT) within LLMs, demonstrating its ability to solve complex tasks like arithmetic and dynamic programming through sequential reasoning process, despite the limitations of bounded-depth Transformers. Recently, \cite{cll+24} shows that Mamba and State-space Models (SSMs) have the same computational limits as Transformers, residing within the $\mathsf{DLOGTIME}$-uniform $\mathsf{TC}^0$ complexity class.

To the best of our knowledge, circuit complexity theory has not yet been used to analyze the computational constraints of Visual AutoRegressive models.

\subsection{Limitation of Transformer Architecture}\label{sec:rw_limitations_of_transformer}
Transformer Architecture has shown remarkable success in various fields, particularly in natural language processing, reinforcement learning, and computer vision. By leveraging self-attention mechanisms to capture long-range dependencies, the Transformer has become the architecture of choice for applications such as machine translation \cite{rt18, wlx+19,yw20} and image generation \cite{pvu+18,dyh+21,tjy+24}. Recently, a series of studies have shed insight into the reasoning limitations of Transformer Architecture \cite{mss22,ms23,fzg24,ms24,wms+24,lss+24,kls+24,hsk+24,chi24,hcl+24,hwl24,hwg+24}. Specifically, \cite{mss22} showed that a generalized form of hard attention can recognize languages that go beyond what the $\mathsf{AC}^0$ class can compute, with the $\mathsf{TC}^{0}$ class serving as an upper bound for the formal languages it can identify. The study by \cite{lag+22} established that softmax-transformers (SMATs) are included in the non-uniform $\mathsf{TC}^{0}$ class. As a next step, \cite{ms23} demonstrated that SMATs belong to $\mathsf{L}$-uniform $\mathsf{TC}^0$ class. Recently, \cite{chi24} demonstrated that average-hard attention transformers (AHATs), without approximation, and SMATs with floating-point precision of $O(\poly(n))$ bits, as well as SMATs with at most $2^{-O(\poly(n)}$ absolute error, can all be classified in the $\mathsf{DLOGTIME}$-uniform $\mathsf{TC}^0$ class.

\subsection{Visual Generation Models}\label{sec:im_gen_models}
Visual generation models have seen significant progress over the past few years, with key advancements in the following mainstream architecture:

\paragraph{AutoRegressive Models.} 
AutoRegressive models for visual generation require the encoding of 2D images into 1D token sequences. Early works in this area, such as PixelCNN \cite{vke+16} and Pixelsnail \cite{cmr+18}, demonstrated the ability to generate pixels in a row-by-row, raster-scan fashion. More later, there are many works that generate image tokens in the raster-scan order, like \cite{rvav19,ero21,lkk+22}. Specifically, VQ-GAN \cite{ero21} uses a GPT-2 decoder-only transformer for image generation. VQVAE-2 \cite{rvav19} and RQ-Transformer \cite{lkk+22} also adopt this raster-scan approach but incorporate additional scales or stacked codes. Recently, \cite{tjy+24} proposed Visual AutoRegressive (VAR) modeling, a new approach to autoregressive image generation, which redefines the process by predicting the next scale in a coarse-to-fine manner. The VAR model not only improves upon traditional methods but also surpasses diffusion transformers in terms of scalability,  inference speed, and image quality.

\paragraph{Diffusion Models.} Diffusion Models have earned recognition for generating high-resolution images by gradually removing noise, exemplified by models such as DiT \cite{px23} and U-ViT \cite{bnx+23}. Typically, these models apply a series of diffusion stages to transform random noise into a coherent image, learning the underlying data distribution through a probabilistic framework. Recent advancements in diffusion-based image generation have focused on improving learning and sampling \cite{se19,sme20,lzb+22,wcz+23,hwl+24,wsd+24,wxz+24,ssz+25_dit,ssz+25_prune,lzw+24}, latent learning \cite{rbl+22,hwsl24} and architecture \cite{hsc+22,px23,xsg+24}.

\section{Preliminary}\label{sec:pre}
The notations used in this paper are introduced in Section~\ref{sec:pre_notation}. Section~\ref{sec:circuit_complexity} explains the basics of circuit complexity classes. Section~\ref{sec:basic_tools} introduces key simulations of floating-point operations, which will be used in later sections for the proofs.

\subsection{Basic Notations}\label{sec:pre_notation}
We apply $[n]$ to represent the set $\{1,2,\cdots, n\}$ for any positive integer $n$. The set of natural numbers is denoted by $\mathbb N := \{0, 1, 2, \ldots\}$. Let $X \in \mathbb{R}^{m \times n}$ be a matrix, where $X_{i,j}$ refers to the element at the $i$-th row and $j$-th column. When $x_i$ belongs to $\{ 0,1 \}^*$, it signifies a binary number with arbitrary length. In a general setting, $x_i$ represents a length $p$ binary string, with each bit taking a value of either 0 or 1.

\subsection{Key Concepts in Circuit Complexity}\label{sec:circuit_complexity}
We discuss several circuit complexity classes, starting with the concept of a boolean circuit.

\begin{definition}[Boolean Circuit, Definition 6.1 in \cite{ab09}]\label{lem:bool_cir}
A Boolean circuit with input size $n$, where $n \in \mathbb{N}$, corresponding to a function that $C_n:\{0,1\}^n \to \{0,1\}$.  This circuit can be typically represented as a directed acyclic graph (DAG). There are $n$ input nodes in the graph, all with an in-degree of $0$. Other nodes are classified as logic gates and are assigned one of the labels $\mathsf{AND}$, $\mathsf{OR}$, or $\mathsf{NOT}$. We use $|C_n|$ to represent the size of $C_n$, referring to the count of nodes in the Boolean circuit.
\end{definition}

Therefore, we can proceed to define the languages recognizable by certain families of Boolean circuits, considering their structural constraints, gate types, and depth. These factors determine the computational power of the circuits in each family.

\begin{definition}[Language, Definition 6.2 in \cite{ab09}] \label{lem:lg_circ}
Let $L \subseteq \{0, 1\}^*$ denote a language. $L$ can be recognized by a Boolean circuits family $\mathcal{C}$ if, for every string $x \in \{0,1\}^{*}$,  a Boolean circuit $C_{|x|} \in \mathcal{C}$ exists, which takes $x$ as input. This circuit has an input length of $|x|$, and $x \in L$ if and only if $C_{|x|}(x) = 1$ holds.
\end{definition}

Next, the concept of complexity classes will be given, which categorizes computational problems based on their inherent difficulty, determined by the resources—such as time or space—required to solve them. In this context, different complexity classes impose constraints on the resources of Boolean circuits, which can be further characterized by factors such as circuit size, depth, number of fan-in, and gate types. We introduce the complexity classes as the following
\begin{itemize}
    \item A language belongs to $\mathsf{NC}^i$ class if it can be decided by a $\poly(n)$ size, $O(\log^{i}(n))$ depth boolean circuits equipped with restricted fan-in basic gates $\mathsf{AND}$, $\mathsf{OR}$ and $\mathsf{NOT}$ gates.
    \item A language belongs to $\mathsf{AC}^i$ class if it can be decided by a $\poly(n)$ size, $O(\log^{i}(n))$ depth boolean circuits equipped with no-limit fan-in basic gates $\mathsf{AND}$, $\mathsf{OR}$ and $\mathsf{NOT}$ gates.
    \item A language belongs to $\mathsf{TC}^i$ class if it can be decided by a $\poly(n)$ size, $O(\log^{i}(n))$ depth boolean circuits equipped with no-limit fan-in basic gates $\mathsf{AND}$, $\mathsf{OR}$, $\mathsf{NOT}$ and $\mathsf{MAJORITY}$ gates.
    \item A language belongs to $\mathsf{P}$ class if it can be decided by a deterministic Turing machine in polynomial time with respect to its input size
\end{itemize}
There is a folklore regarding the hierarchical relationships between the complexity classes mentioned above, for every $i \in \mathbb{N}$: 
\begin{align*}
    \mathsf{NC}^i \subseteq \mathsf{AC}^i \subseteq \mathsf{TC}^i \subseteq \mathsf{NC}^{i+1} \subseteq \mathsf{P}.
\end{align*}
Note that the question of whether $\mathsf{TC}^0 \subsetneq \mathsf{NC}^{1}$ remains an open problem in circuit complexity.

In theoretical computer science, the uniformity of a complexity class refers to whether the circuit family in question can be constructed by a uniform algorithm, i.e., an algorithm that outputs a description of the circuit for any input size. Specifically, $\mathsf{L}$-uniformty requires a Turing machine that uses $O(\log(n))$ space to output a circuit $C$ which can recognize a given language $L \subseteq \{0,1\}^*$. Moreover, $\mathsf{DLOGTIME}$-uniformity stipulates that a random access Turing machine must produce a circuit $C$ that recognizes a given language $L \subseteq \{0,1\}^*$. Except in the case of small circuit complexity classes, where circuits are incapable of simulating the machines that create them, $\mathsf{DLOGTIME}$-uniformity is the same as $\mathsf{L}$-uniformity.  For further discussion on various notions of uniformity, see \cite{bi94, hab02}.

Throughout this work, any reference to a uniform $\mathsf{TC}^0$ should be understood as referring to a $\mathsf{DLOGTIME}$-uniform $\mathsf{TC}^0$.

\subsection{Basic Tools}\label{sec:basic_tools}
In this section, we first define floating-point numbers and then illustrate a series of operations involving them. Finally, we analyze the circuit complexity associated with these operations, which is essential in the later proof.

\begin{definition}[Floating point number, Definition 9 in \cite{chi24}]\label{def:float_point}
    Let $p$ be an integer representing precision. Let $m \in (-2^p,-2^{p-1}] \cup \{0\} \cup [2^{p-1},2^p)$ denote an integer called the significance. Let $e \in  [-2^p, 2^p)$ denote an integer called the exponent. A  floating point number with $p$-bits is composed of the parts $m$ and $e$, and its value is given by $m \cdot 2^e$. Throughout this paper, the set of all $p$-bit floating-point numbers is denoted by $\mathbb{F}_p$.
\end{definition}

Then, we move forward to define the round operation of float point numbers. 
\begin{definition}[Rounding Operation, Definition 9 in \cite{chi24}]
    Given a floating point number $x$, we use $\round_p(x)$ to denote the nearest number to $x$ which is $p$-bit floating-point.
\end{definition}

For the definitions of addition, multiplication, division, comparison, and floor operations on floating-point numbers as outlined in Definition~\ref{def:float_point}, refer to \cite{chi24}. In this paper, we introduce the corresponding circuit complexity classes to which these operations belong.

\begin{lemma}[Operations on floating point numbers in $\mathsf{TC}^0$, Lemma 10  and Lemma 11 of~\cite{chi24}]\label{lem:float_operations_TC} 
Assume the precision $p \leq \poly(n)$. Then we have:
\begin{itemize}
    \item Part 1. Given two $p$-bits float point numbers $x_1$ and $x_2$. Let the addition, division, and multiplication operations of $x_1$ and $x_2$ be outlined in \cite{chi24}. Then, these operations can be simulated by a size bounded by $\poly(n)$ and constant depth bounded by $d_{\rm std}$ $\mathsf{DLOGTIME}$-uniform threshold circuit.
    \item Part 2. Given $n$ $p$-bits float point number $x_1,\dots,x_n$. The iterated multiplication of $x_1, x_2\dots, x_n$ can be simulated by a size bounded by $\poly(n)$ and constant depth bounded by $d_\otimes$ $\mathsf{DLOGTIME}$-uniform threshold circuit.
    \item Part 3. Given $n$ $p$-bits float point number $x_1,\dots,x_n$. The iterated addition of $x_1, x_2\dots, x_n$ can be simulated by a size bounded by $\poly(n)$ and constant depth bounded by $d_\oplus$ $\mathsf{DLOGTIME}$-uniform threshold circuit. To be noticed, there is a rounding operation after the the summation is completed.
\end{itemize}
\end{lemma}

Then, we show a lemma stating that we can use a $\mathsf{TC}^0$ circuit to simulate the approximated exponential function.

\begin{lemma}[Approximating the Exponential Operation in $\mathsf{TC}^0$, Lemma 12  of~\cite{chi24}]\label{lem:exp}
    Assume the precision $p \leq \poly(n)$. Given any number $x$ with $p$-bit float point, the $\exp(x)$ function can be approximated by a uniform threshold circuit. This circuit has a size bounded by $\poly(n)$ and a constant depth $d_{\rm exp}$, and it guarantees a relative error of at most $2^{-p}$. 
\end{lemma}
Finally, we present a lemma stating that we can use a $\mathsf{TC}^0$ circuit to simulate the approximated square root operation.
\begin{lemma}[Approximating the Square Root Operation in $\mathsf{TC}^0$, Lemma 12  of~\cite{chi24}]\label{lem:sqrt}
     Assume the precision $p \leq \poly(n)$. Given any number $x$ with $p$-bit float point, the $\sqrt{x}$ function can be approximated by a uniform threshold circuit. This circuit has a size bounded by $\poly(n)$ and a constant depth $d_\mathrm{sqrt}$, and it guarantees a relative error of at most $2^{-p}$. 
\end{lemma}
\section{Model Formulation}\label{sec:model_formulation}
Section~\ref{def:mf_def} provides the definitions related to the VAR model. In Section~\ref{sec:phase_1}, we present the mathematical formulation of the components involved in the token map generation phase of the VAR model. In Section~\ref{sec:phase_2}, we present the mathematical formulation of the components involved in the feature map reconstruction phase of the VAR model. In Section~\ref{sec:phase_3}, we present the mathematical formulation of the components involved in the VQ-VAE Decoder phase of the VAR model.

\subsection{Definitions}\label{def:mf_def}
We give the following notations in our setting.

\begin{definition}[Codebook of VQ-VAE]\label{def:cb_vae}
In VQ-VAE, the Codebook is typically represented as a matrix $ \mathsf{C} \in \mathbb{R}^{c_{\rm VAE} \times d_{\rm VAE}}$, where:
\begin{itemize}
    \item $c_{\rm vae} $ is the number of vectors in the Codebook (i.e., the size of the Codebook),
    \item $ d_{\rm vae} $ is the dimensionality of each vector.
\end{itemize}
\end{definition}

\subsection{Phase 1: VAR Transformer}\label{sec:phase_1}
VAR uses the VAR Transformer to convert the initial tokens of the generated image into several pyramid-shaped token maps. And in the token maps generation phase, the token maps for the next scale, $M_{k+1}$, are generated based on the previous $k$ token maps ${M_1, \dots, M_{k}}$. This phase has the main modules as the following:

\paragraph{Up Sample Blocks.} 
VAR performs an upsampling operation on the $(i)$-th token map, adjusting its size to that of the $(i+1)$-th token map, before feeding the $k$ token maps into the VAR Transformer. Specifically, VAR employs an upsampling method using interpolation for the image.  Here, we define the up-interpolation blocks:

\begin{definition}[Bicubic Spline Kernel]\label{def:bi_spline_kernel}
    A bicubic spline kernel is a piecewise cubic function $W: \mathbb{F}_p \to \mathbb{F}_p$ that satisfies $W(x) \in [0,1]$ for all $x \in \mathbb{F}_p$.
\end{definition}

\begin{definition}[Up-interpolation Layer]\label{def:up_inter_layer}
The Up-Interpolation layer is defined as follows:
\begin{itemize}
    \item Let $h$ and $h'$ represent the heights of the input and output feature maps, respectively, where $h, h' \in \mathbb{N}$.
    \item Let $w$ and $w'$ denote the widths of the input and output feature maps, respectively, where $w, w' \in \mathbb{N}$.
    \item Let $c \in \mathbb{N}$ denote the number of channels.
    \item Let $X \in \R^{h \times w\times c}$ denote the input feature map.
    \item Let $Y \in \R^{h' \times w' \times c}$ denote the output feature map.
    \item Let $s,t \in \{-1,0,1,2\}$.
    \item Let $W: \mathsf{F}_p^{h \times w \times c} \to \mathsf{F}_p^{h \times w \times c}$ be a bicubic spline kernel as defined in~\ref{def:bi_spline_kernel}.
    
    We use $\phi_{\rm up}: \mathsf{F}_p^{h\times w \times c} \to \mathsf{F}_p^{h' \times w' \times c}$ to denote the up-interpolation operation then we have $Y = \phi_{\rm up}(X)$. Specifically, for $i \in [h'], j \in [w'], l \in [c]$, we have
    \begin{align*}
        Y_{i,j,l} := \sum_{s=-1}^2 \sum_{t=-1}^2 W(s) \cdot X_{\frac{i h}{h'}+s,\frac{j w}{w'}+t,l} \cdot  W(t)
    \end{align*}
    
\end{itemize}
\end{definition}

\paragraph{Transformer Blocks.} 
After the up-sample process, the generated token maps above will be input into the Transformer to predict the next token map. Here, we define several blocks for the VAR Transformer.

Then, we can move forward to define the attention matrix.
\begin{definition}[Attention Matrix]\label{def:attn_matrix}
    We use $W_Q, W_K \in \mathbb{F}_p^{d \times d}$ to denote the weight matrix of the query and key. Let $X \in \mathbb{F}_{p}^{n \times d}$ represent the input of the attention layer. Then, we use $A \in \mathbb{F}_{p}^{n \times n}$ to denote the attention matrix. Specifically, we denote the element of the attention matrix as the following:
    \begin{align*}
    A_{i,j} := & ~\exp(  X_{i,*}   W_Q   W_K^\top   X_{j,*}^\top ).
\end{align*}
\end{definition}

In the next step, we proceed to define the single attention layer.

\begin{definition}[Single attention layer]\label{def:single_layer_transformer}
     We use $W_V \in \mathbb{F}_p^{d \times d}$ to denote the weight matrix of value. Let $X \in \mathbb{F}_{p}^{n \times d}$ represent the input of the attention layer. Let $A$ denote the attention matrix defined in Definition~\ref{def:attn_matrix}. Let $D := \diag( A {\bf 1}_n) $ denote a size $n \times n$ matrix. Then, we use $\mathsf{Attn}$ to denote the attention layer. Specifically, we have
\begin{align*}
    \mathsf{Attn} (X) := & ~ D^{-1} A X W_V .
\end{align*}
\end{definition}

Then, we move forward to define the multilayer perceptron layer. 

\begin{definition}[Multilayer Perceptron layer]\label{def:mlp}
    Given an input matrix $X \in \mathsf{F}_p^{n \times d}$. Let $i \in [n]$. We use $g^{\rm MLP}$ to denote the MLP layer. Specifically, we have 
    \begin{align*}
        g^{\mathrm{MLP}}(X)_{i,*} := W \cdot X_{i,*} + b.
    \end{align*}
\end{definition}

We then proceed to define the layer-wise normalization layer.
\begin{definition}[Layer-wise normalization layer]\label{def:layer_norm}
     Given an input matrix $X \in \mathsf{F}_p^{n \times d}$. Let $i\in [n]$. We use $g^{\rm LN}$ to denote the LN layer. Specifically, we have 
    \begin{align*}
        g^{\mathrm{LN}} (X)_{i,*}  :=  \frac{X_{i,*} - \mu_i}{\sqrt{\sigma_i^2}},
    \end{align*}
    where $\mu_i := \sum_{j=1}^d X_{i,j} / d$, and $\sigma_i^2 := \sum_{j = 1}^d (X_{i,j} - \mu_i)^2 / d$.
\end{definition}

Recall we have defined $\phi_{\mathrm{up}}: \R^{h \times w \times c} \to \R^{h' \times w' \times c}$ in Definition~\ref{def:up_inter_layer}. Since there is no non-linear operation in $\phi_{\mathrm{up}}$, $\phi_{\mathrm{up}}$ is equivalent to a matrix multiplication operation, where the dimension of the matrix is $\R^{h'w' \times hw}$. For simplicity, we view $\phi_{\mathrm{up}}$ as a $\R^{h'w' \times hw}$ dimension matrix in the following proofs.

\begin{remark} [Applying $\phi_{\mathrm{up}}$ on $X \in \R^{n \times d}$]
The actual input of VAR Transformer Layer are $r$ input token maps, $X_1 \in \R^{h_1 \times w_1 \times d}, \ldots, X_r \in \R^{h_r \times w_r \times d}$. We denote them as $X \in \R^{n \times d}$, where $n := \sum_{i = 1}^r h_i w_i$. We denote $\phi_{\mathrm{up}}(X) \in \R^{n' \times d}$ as applying $\phi_{\mathrm{up}}$ to each $X_i \in \R^{h_i \times w_i \times d}$ for $i \in [r]$, where $n' = \sum_{i=1}^r h_i' w_i'$. 
\end{remark}

Then, we can combine multiple attention layers with other components (up-interpolation layers, multilayer perceptron layers, layer-wise normalization layers) to create a complete VAR Transformer architecture.

\begin{definition}[VAR transformer]\label{def:var_transformer}
Assume the VAR transformer has $m$ Transformer layers. At the $i$-th transformer layer, let $g_i$ denote components excluding the attention layer, such as the LN layer or MLP layer. Let $\mathsf{Attn}_i$ stand for the self-attention layer, which is defined in Definition~\ref{def:single_layer_transformer}.  Given  an input token map $X \in \mathbb{F}_p^{1 \times d}$. We define a VAR transformer as the following
\begin{align*}
    \mathsf{TF}(X) := g_m \circ \mathsf{Attn}_m \circ \phi_{\rm up} \dots \circ \phi_{\rm up} \circ g_1 \circ \mathsf{Attn}_1 \circ \phi_{\rm up} (X) ~~ \in \mathbb{F}_{p}^{n \times d},
\end{align*}
In this expression, $\circ$ stands for functional composition.
\end{definition}

\subsection{Phase 2: Feature Map Reconstruction }\label{sec:phase_2}
In phase 2, VAR will transform the generated token maps into feature maps. This phase has the following main modules:

\paragraph{Up Sample Blocks.} 
The VAR performs upsampling on token maps of different sizes, scaling them to the size of the final output feature map. In this process, VAR will use the up-interpolation blocks defined in Definition~\ref{def:up_inter_layer}. To mitigate information loss during token map up-scaling, VAR employs convolution blocks to post-process the up-scaled token maps. We define the convolution blocks as the following:

\begin{definition}[Convolution Block]\label{def:conv_block}
    Let \( X \in \mathbb{F}_p^{h \times w \times c} \) represent a feature map, where \( h \) and \( w \) denote the height and width of the feature map, and \( c \) is the number of input channels. Consider a convolution kernel \( K \in \mathbb{F}^{h_k \times w_k \times c} \), where \( h_k \) and \( w_k \) are the height and width of the kernel, respectively. Assume the input feature map has zero padding, and the stride of the kernel is \( 1 \). The convolution operation performed by this kernel on the input feature map is defined for \( i \in [1, h - h_k + 1] \) and \( j \in [1, w - w_k + 1] \) as the following:
    \begin{align*}
        Y_{i,j} := \sum_{m=1}^{h_k}\sum_{n=1}^{w_k} \sum_{q=1}^{c} X_{i+m-1,j+n-1,q} \cdot K_{m,n,q} + b
    \end{align*}
    where $Y$ is the output feature map of the convolution layer, and $b$ is the bias term of the kernel.
\end{definition}

\subsection{Phase 3: VQ-VAE Decoder process}\label{sec:phase_3}
VAR will use the VQ-VAE Decoder Module to reconstruct the feature map generated in Section~\ref{sec:phase_2} into a new image. The Decoder of VQ-VAE has the following main modules:

\paragraph{ResNet Blocks.} 
In the VQVAE decoder, the ResNet block, which includes two (or more) convolution blocks, plays a crucial role in improving the model's ability to reconstruct high-quality outputs. The convolution blocks help capture spatial hierarchies and patterns in the data, while the residual connections facilitate better gradient flow and allow the model to focus on learning the residuals (differences) between the input and output. The definition of convolution block is given in Definition~\ref{def:conv_block}.

\paragraph{Attention Blocks.} 
The Attention block helps the Decoder fuse information from different locations during the generation process, which can significantly improve the clarity and detail of the generated images. When applied to a feature map, the attention mechanism computes attention scores for all pairs of pixels, capturing their pairwise relationships and dependencies. The definitions of blocks in attention are given in Section~\ref{sec:phase_1}.

\paragraph{Up Sample Blocks.} 
The VQ-VAE decoder uses Up-Sample Blocks to progressively increase the spatial resolution of the latent representation. The Up-Sample Blocks in VQVAE combine up-interpolation and convolution blocks to restore the spatial dimensions of the feature maps, facilitating the reconstruction of the high-resolution output image. The convolution block has already been defined in Definition~\ref{def:conv_block}, and the up-interpolation block has already been defined in Definition~\ref{def:up_inter_layer}.

\section{Complexity of VAR Models}\label{sec:complexity_result}
We present the critical findings on the circuit complexity of crucial operations in the computation of VAR models. In Section~\ref{sec:ui_block_compute}, we analyze the up-interpolation blocks. In Section~\ref{sec:attention_matrix_compute}, we examine the matrix operations. In Section~\ref{sec:single_attention_layer_compute}, we proceed to study the single attention layer. In Section~\ref{sec:common_components_compute}, we move forward to compute the MLP layer and LN layer. In Section~\ref{sec:conv_compute}, we study the convolution layer computation. In Section~\ref{sec:phase_1_tc}, 
We show that we can use a uniform $\mathsf{TC}^0$ circuit to model the VAR Transformer. In Section~\ref{sec:phase_2_tc}, we show that we can use a uniform $\mathsf{TC}^0$ circuit to model the feature map reconstruction layer. In Section~\ref{sec:phase_3_tc}, we show that we can use a uniform $\mathsf{TC}^0$ circuit to model the VQ-VAE Decoder.   Finally, we show our main result in Section~\ref{sec:main_result}.
\subsection{Computing Up Interpolation Blocks}\label{sec:ui_block_compute}
In this section, we can show that the up-interpolation layers can be computed in $\mathsf{TC}^0$.
\begin{lemma}[Up-Interpolation in $\mathsf{TC}^0$]\label{lem:u_i_tc0}
    Let $X \in \mathbb{F}_p^{h \times w \times c}$ denote the origin feature map. Let $Y \in \mathbb{F}_p^{h' \times w' \times c}$ denote the target up-scaled feature map. Assume the precision $p \leq \mathrm{poly}(n)$, $h,h',w,w'c \leq \mathrm{poly}(n)$, then we can simulate the up-interpolation layer in Definition~\ref{def:up_inter_layer} by a size bounded by $\poly(n)$ and $O(1)$ depth uniform threshold circuit.
\end{lemma}
\begin{proof}

    Firstly, we begin to compute every entry in the targeted feature map $Y$. For $i \in [h'], j \in [w'], l \in [c]$, we have
    \begin{align*}
        Y_{i,j,l} = \sum_{s=-1}^2 \sum_{t=-1}^2 W(s) \cdot X_{\frac{i h}{h'}+s,\frac{j w}{w'}+t,q} \cdot W(t)
    \end{align*}
    By using the result of Part 1 of Lemma~\ref{lem:float_operations_TC}, we can apply a constant depth $2d_{\rm std}$ uniform threshold circuit to compute each product $W(s) \cdot X_{\frac{i h} {h'}+s,\frac{j w}{w'}+t,q} \cdot W(t)$. Since the products for different $s$ and $t$ can be parallel computed, the uniform threshold circuit's depth for all products $W(u) \cdot X_{\frac{i h}{h'}+s,\frac{j w}{w'}+t,q}$ stays $2d_{\rm std}$.

    Then, by using the result of Part 3 in Lemma~\ref{lem:float_operations_TC}, we can use a $d_\oplus$ depth uniform threshold circuit to model the sum operation:
    \begin{align*}
        \sum_{s=-1}^2 \sum_{t=-1}^2 W(s) \cdot X_{\frac{i h}{h'}+s,\frac{j w}{w'}+t,q} \cdot W(t)
    \end{align*}
     Hence, the total depth of the circuit required to compute $Y_{i,j,l}$ is $2d_{\rm std} + d_\oplus$. As we can parallel compute $Y_{i,j,q}$, for all $i \in [h']$, $j \in [w']$ , and $l \in [c]$. So the total depth is $2d_{\rm std} + d_\oplus$. The circuit size is $\poly(n)$, which is due to $h', w', c \leq \mathrm{poly}(n)$. Therefore, the entire Up-Interpolation process can be performed by a uniform threshold circuit, where its size is bounded by $\mathrm{poly}(n)$ and its depth remains constant. This concludes the proof.
\end{proof}

\subsection{Computing Attention Matrix}\label{sec:attention_matrix_compute}
Let us begin by recalling that the matrix multiplication of two matrices belongs to $\mathsf{TC}^0$.

\begin{lemma}[Matrix Multiplication  belongs to $\mathsf{TC}^0$ class, Lemma 4.2 in \cite{cll+24}] \label{lem:matrix_multiplication_tc0}
Assume the precision $p \leq \poly(n)$, $n_1, n_2 \leq \poly(n)$, and $d \leq n$. Let $A \in \mathbb{F}_{p}^{n_1 \times d}$ and $B \in \mathbb{F}_p^{d \times n_2}$. Then we can apply a  $\mathsf{DLOGTIME}$-uniform threshold circuit with constant depth $( d_{\rm std}+d_{\oplus} )$ and size bounded by $\poly(n)$ to get the matrix product $AB$.
\end{lemma}

\subsection{Computing Single Attention Layer}\label{sec:single_attention_layer_compute}
Subsequently, matrix operations can be applied to compute the attention matrix.

\begin{lemma}[Attention matrix computation belongs to $\mathsf{TC}^0$ class]
Assume the precision $p \leq \poly(n)$, then we can use a size bounded by $\poly(n)$ and constant depth $3(d_{\rm std} + d_{\oplus}) + d_{\rm exp}$ uniform threshold circuit to compute the attention matrix $A$ defined in Definition~\ref{def:attn_matrix}.

\end{lemma}
\begin{proof}
    Based on Lemma~\ref{lem:matrix_multiplication_tc0}, 
    we can compute the matrix product $W_Q W_K^\top$ by using a size bounded by $\poly(n)$ and constant depth $d_{\rm std} + d_{\oplus}$ uniform threshold circuit.

    Then, we move forward to compute the scalar product, which is
    \begin{align*}
        s_{i,j} = X_{i,*}W_Q W_K^\top X_{j,*}^\top
    \end{align*}
    And by using the result of Lemma~\ref{lem:matrix_multiplication_tc0}, we can compute $t_{i,j}$ by applying a uniform threshold circuit, where the circuit has a polynomial-size bounded by  $\poly(n)$ and constant depth $2(d_{\rm std} + d_{\oplus})$.

    In the next step, from Lemma~\ref{lem:exp}, we can compute the exponential function $A_{i,j} = \exp(t_{i,j})$ by applying a size bounded by $\poly(n)$ and constant depth $d_{\rm exp}$ uniform threshold circuit.

   After combining depths from all steps, the total depth of the circuit for computing $A_{i,j}$ is 
    \begin{align*}
        d_{\rm total} = 3(d_{\rm std} + d_{\oplus}) + d_{\rm exp}.
    \end{align*}

    Since we can parallel compute all entries in $A_{i,j}$ for $i,j \in [n]$, the circuit depth remains $3(d_{\rm std} + d_{\oplus}) + d_{\rm exp}$ and size bounded by $\poly(n)$.

    Thus, we have proven the result.
\end{proof}

\subsection{Computing Common Components Layers}\label{sec:common_components_compute}
This section outlines the MLP layer circuit complexity.
\begin{lemma}[MLP computation falls within $\mathsf{TC}^0$ class, Lemma 4.5 of \cite{cll+24}]\label{lem:mlp_tc0}
    Assume the precision $p \leq \poly(n)$. Then, we can use a size bounded by $\poly(n)$ and constant depth $2d_\mathrm{std} + d_{\oplus}$ uniform threshold circuit to simulate the MLP layer in Definition~\ref{def:mlp}.
\end{lemma}

Next, we examine the layer-normalization (LN) layer circuit complexity.

\begin{lemma}[LN computation falls within $\mathsf{TC}^0$ class, Lemma 4.6 of \cite{cll+24}]\label{lem:layer_tc0}
    Assume the precision $p \leq \poly(n)$, then we can use a size bounded by $\poly(n)$ and constant depth $5d_\mathrm{std} + 2d_{\oplus} + d_\mathrm{sqrt}$ uniform threshold circuit to simulate the Layer-wise Normalization layer defined in Definition~\ref{def:layer_norm}.
\end{lemma}

\subsection{Computing Convolution Blocks} \label{sec:conv_compute}
We prove in this section that the convolution layers can be computed within $\mathsf{TC}^0$.

\begin{lemma}[One Kernel Convolution Process in $\mathsf{TC}^0$] \label{lem:conv_tc0}
Under the premise that the following conditions apply:
\begin{itemize}
    \item Let $X \in \mathbb{F}_p^{h \times w \times c}$ denote the origin feature map.
    \item Let $K \in \mathbb{F}_p^{h_k \times w_k \times c}$ denote a convolution kernel.
    \item Assume the padding in the convolution process is $0$.
    \item Assume the stride in the convolution process is $1$.
    \item Let $Y \in \mathbb{F}_p^{(h-h_k+1) \times (w-w_k+1) \times c}$ denote the output feature map.
    \item For $i \in [h-h_k+1]$ and $j \in [w-w_k+1]$.
    \item Let $h,w,c \leq \mathrm{poly}(n)$.
    \item Let $h_k \times w_k \times c \leq n$.
\end{itemize}
Then, we can apply a size bounded by $\poly(n)$ and $O(1)$ depth uniform threshold circuit to simulate one kernel convolution process.

\begin{proof}
For each $i \in [h-h_k+1]$ and $j \in [w-w_k+1]$, we know
\begin{align*}
     Y_{i,j} := \sum_{m=1}^{h_k}\sum_{n=1}^{w_k} \sum_{q=1}^{c} X_{i+m-1,j+n-1,q} \cdot K_{m,n,q} + b
\end{align*}
By using the result of Part 1 in Lemma~\ref{lem:float_operations_TC}, we can use a size bounded by $\poly(n)$ and $O(1)$ depth uniform threshold circuit to compute each product $X_{i+m-1,j+n-1,q} \cdot K_{m,n,q}$. Furthermore, the computation of $X_{i+m-1,j+n-1,q} \cdot K_{m,n,q}$ can be performed in parallel for all $m \in [h_k]$, $n \in [w_k]$ and $q \in [c]$. Therefore, the total depth of the circuit remains $O(1)$, and its size stays $\poly(n)$, since $h_k \times w_k \times c \leq n$.

Then, we proceed to compute the sum $\sum_{m=1}^{h_k}\sum_{n=1}^{w_k} \sum_{q=1}^{c} X_{i+m-1,j+n-1,q} \cdot K_{m,n,q}+b$. Using the result from Lemma~\ref{lem:float_operations_TC}, we can use a size bounded by $\poly(n)$ and $O(1)$ depth uniform threshold circuit to compute the sum. By computing $Y_{i,j}$ for all $i \in [h-h_k+1], j \in [w-w_k+1]$ in parallel, we maintain the uniform threshold circuit with $O(1)$ depth and size bounded by $\poly(n)$ which is due to $h,w \leq \poly(n)$.

Thus, we can apply a size bounded by $\poly(n)$ and $O(1)$ depth uniform threshold circuit to simulate the one kernel convolution process.
\end{proof}
\end{lemma}

\begin{proposition}[Multiple Kernel Convolution Process in $\mathsf{TC}^0$]\label{pro:mul_kernel_conv_proc} Assume we have $k$ convolution kernel in a convolution block. Let $k \leq \mathrm{poly}(n)$. Since the computations of different convolutional kernels can be parallelizable, then we can apply a size $\poly(n)$ and $O(1)$ depth to simulate the whole process.
\end{proposition}
\begin{proof}
    This is can be easily derived from  Lemma~\ref{lem:conv_tc0} and $k \leq \mathrm{poly}(n)$.
\end{proof}

\subsection{Computing Phase 1: VAR Transformer}\label{sec:phase_1_tc}
In this part, we establish that the VAR Transformer defined in Definition~\ref{def:var_transformer} is within the computational power of $\mathsf{TC}^0$

\begin{lemma}[VAR Transformer computation in $\mathsf{TC}^0$]\label{lem:var_transformer_tc0}
Assume the number of transformer layers $m = O(1)$. Assume the precision $p \leq \poly(n)$. Then, we can apply a uniform threshold circuit to simulate the VAR Transformer $\mathsf{TF}$ defined in Definition~\ref{def:var_transformer}. The circuit has size $\poly(n)$ and $O(1)$ depth.
\end{lemma}
\begin{proof}
    By using the result of Lemma~\ref{lem:u_i_tc0}, we can apply a uniform threshold circuit of size bounded by $\poly(n)$ and $O(1)$ depth to simulate up -interpolation layer $\phi_{\rm up}$ defined in Definition~\ref{def:up_inter_layer}.

    By using the result of Lemma~\ref{lem:mlp_tc0} and Lemma~\ref{lem:layer_tc0}, we can apply a size bounded by $\poly(n)$ and $O(1)$ depth uniform threshold circuit to simulate $g_i$, for each $i \in [m]$.
    
    By using the result of Lemma~\ref{sec:single_attention_layer_compute}, we can apply a size bounded by $\poly(n)$ and $O(1)$ depth uniform threshold circuit to simulate $\mathsf{Attn}_i$ defined in Definition~\ref{def:single_layer_transformer}.

    To compute $\mathsf{TF}(X)$, we must compute $g_1,\dots,g_m$ ,$\mathsf{Attn}_1,\dots,\mathsf{Attn}_m$ and $m$ up-interpolation layers. Then, we can have that the size of the uniform threshold circuit is bounded by $\poly(n)$, and the total depth of the circuit is $O(1)$, which is due to $m = O(1)$.

    Thus, we complete the proof.
\end{proof}

\subsection{Computing Phase 2: Feature Map Reconstruction }\label{sec:phase_2_tc}
In this section, we show that the feature map reconstruction is within the computational power of $\mathsf{TC}^0$.
\begin{lemma}[Feature Map Reconstruction computation in $\mathsf{TC}^0$.]\label{lem:feature_map_reconstruct_tc0}
Assume the feature map reconstruction needs $k$ convolution kernel. Assume $k\leq \poly(n)$. Assuming the precision $p \leq \poly(n)$, then we can apply a uniform threshold circuit to simulate the feature map reconstruction operations. The circuit has size $\poly(n)$ and $O(1)$ depth.
\end{lemma}
\begin{proof}
    This can be easily derived from Proposition~\ref{pro:mul_kernel_conv_proc}.
\end{proof}

\subsection{Computing Phase 3: VQ-VAE Decoder process}\label{sec:phase_3_tc}
In this section, we show that the VQ-VAE Decoder is within the computational power of $\mathsf{TC}^0$
\begin{lemma}[VQ-VAE Decoder process in $\mathsf{TC}^0$.]\label{lem:vqvae_tc0}
Assume the precision $p \leq \poly(n)$. Then, we can apply a uniform threshold circuit to simulate the VQ-VAE decoder process. The circuit has size $\poly(n)$ and $O(1)$ depth.

\end{lemma}
\begin{proof}
    Firstly, by using the result of  Proposition~\ref{pro:mul_kernel_conv_proc} and Lemma~\ref{lem:float_operations_TC}, we can simulate the ResNet blocks by using a size $\poly(n)$ and $O(1)$ depth uniform threshold circuit.

    Then, by using the result of Lemma~\ref{lem:var_transformer_tc0}, we can simulate the attention blocks by using a size $\poly(n)$ and $O(1)$ depth uniform threshold circuit.

    And, by using the result of Lemma~\ref{lem:u_i_tc0}, 
    we can simulate the Up Sample Blocks by using a size $\poly(n)$ and depth $O(1)$ uniform threshold circuit.

    By combing the result above, we have that a size $\poly(n)$ and $O(1)$ depth uniform threshold circuit can be applied to simulate the VQ-VAE decoder process.
\end{proof}

\subsection{Main Result}\label{sec:main_result}
We present our main result, which derives the circuit complexity limits for the VAR model.

\begin{theorem}[Circuit complexity of the VAR model.]\label{thm:main_theorem}
    Assuming precision $p \leq \poly(n)$, then we can apply a uniform threshold circuit to simulate the VAR model, where the circuit has size $\poly(n)$ and $O(1)$ depth.
\end{theorem}
\begin{proof}
    This result directly comes from  Lemma~\ref{lem:var_transformer_tc0}, Lemma~\ref{lem:feature_map_reconstruct_tc0} and Lemma~\ref{lem:vqvae_tc0}.
\end{proof}
\section{Conclusion}\label{sec:conclusion}
This study provides a comprehensive theoretical analysis of VAR models, deriving key limits on their computational abilities. Our approach centers on examining the circuit complexity of various components of VAR models, from the up-interpolation layers and the convolution layers to the attention mechanism. Furthermore, we show that VAR can be expressed as uniform $\mathsf{TC^0}$ circuits. This finding is important because it exposes inherent constraints in the expressiveness of VAR models, despite their empirical effectiveness in visual generation.
\ifdefined\isarxiv
\bibliographystyle{alpha}
\bibliography{ref}
\else
\bibliography{ref}
\bibliographystyle{alpha}

\fi





\end{document}